\documentclass[preprints,article,accept,moreauthors,pdftex]{Definitions/mdpi} 

\usepackage{algorithmicx}
\usepackage[ruled]{algorithm}
\usepackage{algorithm}
\usepackage{algpseudocode}

\usepackage{glossaries}
\newacronym{mavs}{MAVs}{Micro Aerial Vehicles}
\newacronym{mav}{MAV}{Micro Aerial Vehicle}

\usepackage{amsthm}
\newtheorem{thm}{Theorem}[section]
\newtheorem{lemma}[thm]{Lemma}

\usepackage{hyperref}
\hypersetup{
  colorlinks=true,
  linkcolor=black,
  urlcolor=blue!70!black,
  citecolor=blue!70!black,
}

\firstpage{1} 
\pubvolume{1}
\issuenum{1}
\articlenumber{0}
\pubyear{2021}
\copyrightyear{2021}
\datereceived{} 
\dateaccepted{} 
\datepublished{} 
\hreflink{https://doi.org/} 

\makeatletter
\let\c@lofdepth\relax
\let\c@lotdepth\relax
\makeatother
\usepackage{subfigure}

\Title{Error Tolerant Path Planning for Swarms of Micro Aerial Vehicles with Quality Amplification}

\Author{Michel Barbeau $^{1}$*, Joaquin
  Garcia-Alfaro $^{2}$*, Evangelos
  Kranakis$^{1}$*, and Fillipe Santos$^{3}$}

\address{%
$^{1}$ \quad School of Computer Science, Carleton University, Ottawa, Canada
\\
$^{2}$ \quad Institut Polytechnique de Paris, Telecom SudParis,  France
\\
$^{3}$ \quad University of Campinas, Brazil\\
}

\TitleCitation{Error Tolerant Path Planning for Swarms of Micro Aerial Vehicles with Quality Amplification}

\AuthorNames{Michel Barbeau, Joaquin Garcia-Alfaro, Evangelos Kranakis, and Fillipe Santos}

\AuthorCitation{Barbeau, M.; Garcia-Alfaro, J.; Kranakis, E.; Santos, F.}

\corres{Correspondence: barbeau@scs.carleton.ca (M.B.) jgalfaro@ieee.org (J.G-A.) kranakis@scs.carleton.ca (E.K.); Tel.: +1-613-520-2600x1644 (M.B.) 
+33-160-76-47-22 (J.G-A.) +1-613-520-2600x8090 (E.K.)}

\firstnote{These authors contributed equally to this work.} 

\abstract{We present an error tolerant path planning algorithm for \gls*{mav}
swarms. We assume navigation without
GPS-like techniques. The MAVs find their path
using sensors and cameras, identifying and following
a series of visual landmarks. The visual landmarks lead the MAVs
towards their destination. MAVs are assumed to be unaware
of the terrain and locations of the landmarks. 
They hold a-priori information about landmarks, whose interpretation is prone to errors. 
Errors are of two types,
{\em recognition} or {\em advice}. Recognition errors follow from misinterpretation of sensed data or a priori information, or
confusion of objects, e.g., due to faulty sensors. Advice errors are consequences of outdated or wrong information about landmarks,
e.g., due to weather conditions. Our path planning algorithm is cooperative. MAVs communicate and exchange
information wirelessly, to minimize the number of {\em recognition} and {\em advice} errors. Hence, the quality of the navigation decision process is amplified. Our solution successfully achieves an adaptive error tolerant
navigation system. Quality amplification is parametetrized with respect
to the number of MAVs. We validate our approach with
theoretical proofs and numeric simulations.}


\begin{document}

\section{Introduction}

\noindent \gls*{mavs} are a popular type of drones. They are equipped with
sensors and cameras, enabling hovering and navigation over complex
three dimensional terrains. They are used in a variety of
applications, including sewer inspection \cite{sewercopters}, search
and rescue operations \cite{searchRescue}, and parcel delivery
\cite{dhlParcelcopter}. Large terrains can be covered by so called
swarms, namely collaborative teams of \gls*{mavs} that exchange
information gathered during navigation. They are required to be
resilient to failures of all kinds, such as during navigation
or due to sensor malfunctions. We are interested in designing
swarm algorithms that are resilient in presence of failures.

We present an error tolerant path planning algorithm for \gls*{mav}
swarms. We assume \gls*{mav} navigation without using any
GPS-like technique. \gls*{mavs} find their way using
 sensors and cameras, in order to identify and follow
series of visual landmarks. The visual landmarks lead the \gls*{mavs}
towards the destinations. We assume that the \gls*{mavs} are
unaware of the terrain and locations of the landmarks. Figure
\ref{fig:figure1} shows the idea. 
A number of landmarks are highlighted.
A swarm of \gls*{mavs} collectively
identify a series of such landmarks over an inspected terrain. The
identification of landmarks determines paths that must be followed.

\begin{figure}[!t]
\centering
\includegraphics[width=14cm]{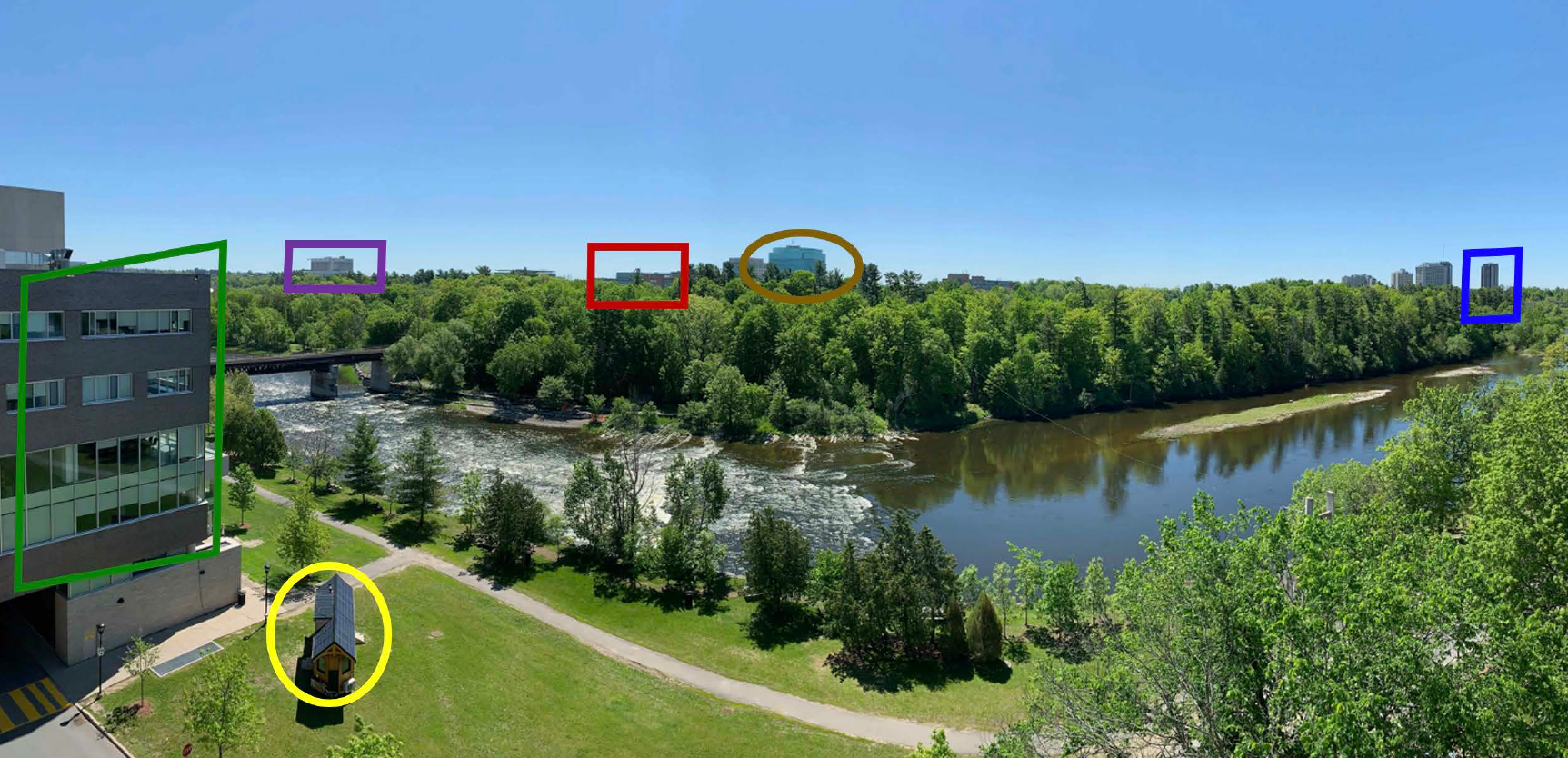}
\caption{Sample picture taken from an aerial vehicle, together
with the identification of six landmarks at the campus of Carleton University.\label{fig:figure1}}
\end{figure}

Swarms fly over terrains
comprising multiple landmarks. The landmarks
also include the starting and terminal points of a well-defined path.
The \gls*{mavs} may hover over the landmarks either on their own or in
formation. They may hop from anyone landmark to any other. The
landmarks are identified as vertices.
The resulting
system forms a complete graph. Recall that the \gls*{mavs} are
unaware of the terrain and locations of the landmarks. However, they
have the capability to visually recognize them. Furthermore, the
\gls*{mavs} may communicate and exchange information wirelessly as
long as they are within communication range of each other.  The
\gls*{mavs} are required to find a flight path from the starting
point, leading to the terminal point.

We assume that the \gls*{mavs} hold information about landmarks.
Interpretation of this information is error prone. We consider
two types of errors: {\em recognition} and {\em advice}.
Recognition errors are due to misinterpretation of sensed data or
a priori information, or confusion of objects, e.g., due to faulty
sensors. Advice errors follow from changing or wrong information
associated to landmarks, e.g., due to weather conditions. Path
planning builds upon swarm cooperation. \gls*{mavs}
communicate and exchange information wirelessly, with the aim to reduce the amount of {\em recognition} and {\em advice} errors.  Collaboratively exchanging information, the \gls*{mavs}
amplify the quality of decisions pertaining to the navigation process. The swarm gets equipped with an adaptive error
tolerant navigation, where the degree of quality is related to the number of participating \gls*{mavs}.

We show that the approach augments the probability of navigation
correctness proportionally to the number of \gls*{mavs} in a swarm,
 when  wireless communications allow
cooperation and exchange of information about landmarks. Indeed, single \gls*{mav} navigation is directly affected by
recognition and advice  errors. 
The \gls*{mav} can get disrupted and lost. 
With an increasing number of \gls*{mavs} in a swarm, communications and exchange of information 
take place. Quality of sensor fusion increases. We analyze  reduction of error probability induced by this
algorithm. Quality amplification is demonstrated both
analytically and with simulation.

\medskip

\noindent \textbf{Paper Organization\footnote{This a revised and extended version of a paper~\cite{barbeau2019quality} which appeared in the proceedings of IEEE GLOBECOM 2019 Workshops: IEEE GLOBECOM 2019 Workshop on Computing-Centric Drone Networks in Waikoloa, Hawaii, Dec 9-14, 2019.} ---} Section~\ref{sec:relatedWork} 
reviews related work.
Sections~\ref{sec:algo} and~\ref{quality:sec} present
our navigation algorithm. Section~\ref{sec:sim} evaluates the work. Section~\ref{sec:conc} concludes the paper.

\section{Related Work}
\label{sec:relatedWork}

\noindent Surveys on path planning algorithms for unmanned aerial 
vehicles have
been authored by Goerzen et al.~\cite{Goerzen2009} and Radmanesh et
al.~\cite{Mohammadreza2018}. Several algorithms build on solutions
originally created for computer networks. Some of the proposed
solutions leverage algorithms created in the field of classical
robotics, such as approaches using artificial potential
functions~\cite{khatib1986}, random trees~\cite{lavalle1998} or
Voronoi diagrams~\cite{lavalle2006}. Path planning may be addressed in
conjunction with team work and formation control~\cite{turpin2014}.
There are ideas that have been tailored specifically to
quadcopters~\cite{rizqi2014}.

Our research is closely related to works on navigation using
topological maps~\cite{Maravall2017}. Navigation does not rely on
coordinates. The \gls*{mavs} find their way recognizing landmarks.
Weinstein et al. \cite{Weinstein2018visual} propose the use of visual
odometry as an alternative localization technique to, e.g., GPS-like
techniques. The idea is as follows. The \gls*{mavs} use their onboard
cameras (e.g., downward facing cameras), combined by some inertial
sensors, to identify and follow a series of visual landmarks. The
visual landwarks lead the \gls*{mav} towards the target destination.
Unlike GPS, the technique allows the \gls*{mav} to operate without
boundaries in both indoor and outdoor environments. No precise
information about concrete visual odometry techniques are reported by
Weinstein et al. in their work. However, some ideas can be found in
\cite{Maravall2013a,Maravall2017}.

Maravall et al. \cite{Maravall2013a,Maravall2017} propose the use of
probabilistic knowledge-based classification and learning automata for
the automatic recognition of patterns associated to the visual
landmarks that must be identified by the \gls*{mavs}. A series of
classification rules in their conjunctive normal form (CNF) are
associated to a series of probability weights that are adapted
dynamically using supervised reinforcement learning
\cite{narendra2012learning}. The adaptation process is conducted using
a two-stage learning procedure. During the first process, a series of
variables are associated to each rule. For instance, the variables
associated to the construction of a landmark recognition classifier
are constructed using images' histogram features, such as standard
deviation, skewness, kurtosis, uniformity and entropy. During the
second process, a series of weights are associated to every variable.
Weights are obtained by applying a reinforcement algorithm, i.e.,
incremental R-L algorithm in
\cite{narendra2012learning,Maravall2017}, over a random environment.
As a result, the authors obtain a specific image classifier for the
recognition of landmarks, which is then loaded to the \gls*{mavs}.

The resulting classifiers had been tested via experimental work.
\gls*{mavs} with high-definition cameras, recording images at a
resolution of $640$x$360$ pixels, at the speed of $30$ fps (frames per
second) are loaded a given classifier, to evaluate a visual
classification ratio. Each experiment consists of building a
classifier and getting the averaged ratio. Results by Maravall et al.
in \cite{Maravall2013a,Maravall2013b} show an average empirical visual
error ratio of about $20$\% (i.e., $80$\% chances of properly
identifying the landmarks, on average). The results are compared to
some other well-established pattern recognition methods for the visual
identification of objects, such as minimum distance and $k$-nearest
neighbor classification algorithms. The previous contribution is complemented by Fuentes et al. and
Maravall et al. in \cite{Fuentes2014,Maravall2017}, by combining the
probabilistic knowledge-based classifiers with bug algorithms
\cite{lavalle2006}, to provide the \gls*{mavs} with a navigation
technique to traverse a visual topological map composed of several
visual landmarks. A technique is used to compute the entropy of the
images captured by the \gls*{mav}, in case a decision must be taken
(e.g., to decide whether going south or north directions). The idea is
as follows. The \gls*{mav} uses the camera onboard, and takes images
about several directions. Afterward, it processes the images to chose
a given direction. The lower the entropy of a captured image, the
lower the probability of going towards an area containing visual
landmarks. Conversely, the higher the entropy of a captured image, the
higher the probability of going towards an area surrounded by
landmarks. Using this heuristic, the \gls*{mav} collects candidate
images with maximum entropy (e.g., by driving the \gls*{mav} forward
and backward some meters) prior executing a bug algorithm to locate
the landmarks \cite{Maravall2017}. 

\section{Error Prone Navigation}
\label{sec:algo}

\noindent We identify the landmarks with the $n$ vertices of a complete 
graph $G=(V,E)$.
Starting at $s$ and ending at $t$, the \gls*{mavs} are seeking a flight
path connecting $k+1$ vertices
$$
s:=v_0, v_1, \ldots, v_i, v_{i+1}, \ldots, v_k:=t
$$
where $v_0, v_1, \ldots, v_i, v_{i+1}, \ldots, v_k$ are in $V$, see
Figure~\ref{fig:net2}. The \gls*{mavs} have to navigate and find a flight
path from $s$ to $t$ using clues. When hovering over an area, a \gls*{mav}
acquires data through its camera and other sensors, which may be
visual, acoustic, etc. This data is used for landmark searching. A
priori, the \gls*{mavs} are given clues and specific characteristics about
the landmarks. For example, the \gls*{mavs} may be seeking a green door or
a tall building.

\begin{figure}[!t]
	\begin{center}
		\includegraphics[width=11cm]{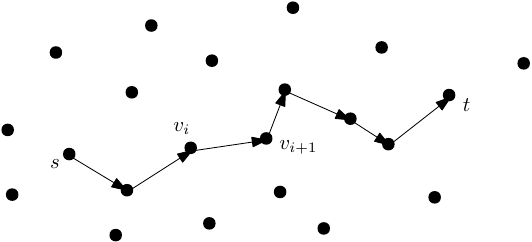}
	\end{center}
	\caption{Flight path from source $s$ to destination $t$. Edge
          $(v_i, v_{i+1})$ is an intermediate segment connecting
          landmarks $v_i$ and $v_{i+1}$.}
	\label{fig:net2}
\end{figure}

The landmarks provided have a-priori information whose interpretation
(by the \gls*{mavs}) is prone to errors. We distinguish two types of
errors, namely, {\em recognition} and {\em advice}. Recognition errors
are due to misinterpretation of sensed data and a-priori information
or confusion of objects. For example, a \gls*{mav} has found a green door
which in fact is not a door but rather a window. The recognized object
is incorrect. We assume that for some real number $p$ in the interval
$[0,1]$, the value $p$ is the probability that a \gls*{mav} performs
recognition erroneously and $1-p$ that it is correct.

Advice errors about landmarks occur because the information provided
is not up to date or even wrong. For example, upon finding a landmark
a \gls*{mav} is advised to traverse a certain distance within the terrain
in direction north where it will find the next landmark, say a
restaurant, but this information is wrong because the restaurant is no
longer there. We assume that for some real number $q$ in the interval
$[0,1]$, the value $q$ is the probability that the advice provided to
a \gls*{mav} about a landmark is invalid or erroneously interpreted and
$1-q$ that it is valid and correctly interpreted.

Recognition and advice errors are independent of each other. An
important point to be made is that we assume that recognition and
advice are random processes. For all \gls*{mavs}, we make the assumption
that recognition errors are independent and identically distributed
and advice errors are also independent and identically distributed.
The \gls*{mavs} act independently of each other. Moreover, the outcome of
the recognition process is random with probability of success that
depends on the parameter $p$. A similar observation applies to the
advice process. As a consequence, we can use this to our advantage so
as to improve the recognition and advice mechanisms for swarms of
\gls*{mavs}.

Assume a \gls*{mav} is navigating the terrain through a flight path,
denoted as $P$, consisting of $k$ vertices $v_0 :=s, v_1, \ldots, v_i,
v_{i+1}, \ldots, v_k:= t$ from $s$ to $t$. An edge $\{ v_i, v_{i+1}
\}$ corresponding to a segment of flight path $P$ is said to be
correctly traversed if and only if the advice provided about the landmark
associated with vertex $v_i$ is valid and correctly interpreted and
the landmark associated with vertex $v_{i+1}$ is correctly recognized.
For $i = 0, \ldots , k-1$, the flight path $P$ is correctly traversed
if and only if each of its segment defined by an edge $\{ v_i, v_{i+1}
\}$ is correctly traversed.

At the start, a \gls*{mav} is given a flight plan. The flight plan defines
the flight path $P$. For each vertex $v_i$,
$i=0,\ldots,k-1$, the flight plan comprises advice for searching the
next landmark, such as directional data. For each vertex $v_{i+1}$,
the flight plan contains recognition data, such as landmark
characteristics. A flight plan is correctly performed solely if every
single segment is correctly traversed.

\begin{algorithm}[!t]
	\setlength\baselineskip{18pt}
	\caption{Majority Recognition Algorithm for a swarm of $m$ \gls*{mavs}}\label{alg:rec}
	\begin{algorithmic}[1]
		\State  {Each \gls*{mav} performs landmark recognition}
		\State {\gls*{mavs} exchange information}
		\If {there is a landmark common to the majority (of at least $\lceil m/2 \rceil$ \gls*{mavs})}
		\State {the \gls*{mav} swarm adopts this common landmark}
		\Else
		\State {every \gls*{mav} adopts its own recognized landmark}
		\EndIf
	\end{algorithmic}
\end{algorithm}

\begin{algorithm}[!t]
\setlength\baselineskip{18pt}
\caption{Majority Advice Algorithm for a swarm of $m$ \gls*{mavs}}\label{alg:adv}
\begin{algorithmic}[1]
\State  {Each \gls*{mav} takes the advice provided for the  visited landmark}
\State {\gls*{mavs} exchange information}
\If {there is a majority advice interpretation (for at least $\lceil m/2 \rceil$ \gls*{mavs})}
\State {all \gls*{mavs} follow this common advice interpretation}
\Else
\State {the \gls*{mavs} follow their own advice interpretation}
\EndIf
\end{algorithmic}
\end{algorithm}

We obtain the following quantitative characterization of segment correctness and flight path in terms of recognition and advice probabilities.

\begin{lemma}\label{lemma:singledroneflightpath}
A flight plan leading to a path of length $k$ is correctly performed with probability $(1-p)^k (1-q)^k$.
\end{lemma}

\begin{proof}
For individual segments $i = 0, \ldots , k-1$, we have
\begin{align*}
\Pr [ \{ v_i, v_{i+1}\} \mbox{ is correct}]
&= \Pr [ \mbox{advice at $v_i$ and recognition} \\
\mbox{at $v_{i+1} $ are correct}] & = (1-p) (1-q).
\end{align*}
For the whole flight plan for path $P$, we have
\begin{align*}
\Pr [ P \mbox{ is correct}]
&= \Pr [ \forall i (\{ v_i, v_{i+1}\}  \mbox{ is correct}) ]\\
&= \prod_{i=0}^{k-1} \Pr [ \{ v_i, v_{i+1}\}  \mbox{ is correct} ]\\
&= (1-p)^k (1-q)^k.
\end{align*}
This proves the lemma.
\end{proof}

Lemma~\ref{lemma:singledroneflightpath} is valid for a single \gls*{mav}
that is recognizing landmarks and navigating from a start point to a
terminal point. In Section~\ref{quality:sec} it is shown how to improve
the probability of correctness for a swarm of co-operating \gls*{mavs} that
communicate and exchange information with each other.

In a swarm, we may take advantage of
communications and collaboration among the \gls*{mavs} so as to amplify the
quality of a-priori and sensed data. To this end, we use the principle
of maximum likelihood.

Algorithms~\ref{alg:rec} and~\ref{alg:adv} define the main processes. Algorithm~\ref{alg:rec} applies majority recognition. Algorithm~\ref{alg:adv}
applies the advice. It should be emphasized that the amplification of recognition and advice, implied by the majority rule used in the two algorithms above, is based on a binary decision. To illustrate this fact, consider the case of amplification of the quality of recognition. First of all, it is assumed that all the \gls*{mavs} in the swarm run the same visual recognition software. Hence, the set of possible outcomes of the \gls*{mavs}' visual systems is partitioned into two mutually disjoint sets. The first set can be interpreted as the container of positive outcomes. The second set as the container of negative outcomes. This is to be the same for all the \gls*{mavs}. For a binary decision example, consider a swarm of five \gls*{mavs} which is to decide whether the object viewed is either a Door (D) or a Window (W). If the answers of the individual \gls*{mavs} are D, W, D, W, D, then the majority output will be Door.

A similar interpretation is being used for the advice algorithm software which is executed by ``smart landmarks'' giving advice to the \gls*{mavs}, i.e., providing the direction the swarm should follow next. For a binary example with a swarm of five \gls*{mavs}, assume that the landmarks may give either the answer North (N) or South (S). If the advice collected by the \gls*{mavs} are N, S, S, N, N, then the majority decision will be North.

\section{Quality Amplification and Error Reduction}
\label{quality:sec}

\subsection{Reducing the error probability}

The collaborative landmark recognition process defined by
Algorithm~\ref{alg:rec} applies to a swarm composed of $m$
\gls*{mavs}.
Let $p_m$ denote the  error probability of the majority rule applied in
Algorithm~\ref{alg:rec}; this is given by the following formula.
\begin{align}
p_m &= 1 - \label{majorityfmla}
\sum_{i= \lceil m/2 \rceil }^{m} {m \choose i} (1-p)^i p^{m-i}
\end{align}
Now we show that the majority rule improves the error probability $p$.
\begin{lemma}\label{lemma:plowerthanonehalf}
For $p < 1/2$, we have the following inequality
\begin{align}
1 - p &<  \label{maj:rec1}
p^{m} \sum_{i= \lceil m/2 \rceil }^{m}  {m \choose i} \left(\frac1p -1 \right)^i  .
\end{align}
\end{lemma}

\begin{proof} (Lemma~\ref{lemma:plowerthanonehalf})
The inequality is proved by considering two cases depending on the parity of $m$, the number of \gls*{mavs}.

\medskip

\noindent {\bf Case 1: $m$ is odd.} If $m$ is odd, we can express the value as $m = 2d+1$, for some integer $d \geq 1$ so that $\lceil m/2 \rceil = d+1$. Let $a = \frac 1p -1$ and observe that $a>1$, since $p < \frac 12$.
From the binomial theorem we have that
\begin{align}
(a+1)^{m}
&= \notag \sum_{i =0 }^{m} {m \choose i} a^i \\
&= \notag \sum_{i=0 }^{d}  {m \choose i} a^i  +
\sum_{i = d+1 }^{m} {m \choose i} a^i \\
&= \label{maj:rec2} L+U,
\end{align}
where $L$ and $U$ are defined as follows
\begin{align}
L & :=  \label{compare0:eq}
\sum_{i=0 }^{d}  {m \choose i} a^i = \sum_{i=0 }^{d} {m \choose d-i} a^{d-i},  \mbox{ and }\\
U &:= \label{compare:eq}
\sum_{i = d+1 }^{m} {m \choose i} a^i \ = \sum_{i=0 }^{d} {m \choose d+i+1} a^{d+i+1}.
\end{align}

Now observe that $L$ and $U$ have the same number of summands with identical respective binomial coeficients, namely
$$
{m \choose d-i} a^{d-i} \mbox{ and } {m\choose d+i+1} a^{d+i+1} ,
$$
for $i=0,1,\ldots ,d$. In Formulas~\eqref{compare0:eq}-\eqref{compare:eq} observe that the left term when multiplied by
$a^{2i+1}$ is equal to the right term, namely $a^{2i+1} {m \choose d-i} a^{d-i} =  {m\choose d+i+1} a^{d+i+1}$, for $i=0,1,\ldots ,d$.
Since $a> 1$ and $d \geq 1$ we conclude that
\begin{align}
aL &= \label{compare1:eq}
\sum_{i=0 }^{d} a {m \choose d-i} a^{d-i} < \sum_{i=0 }^{d} a^{2i+1} {m \choose d-i} a^{d-i} = U.
\end{align}
From Equations~\eqref{maj:rec2}~and~\eqref{compare1:eq}, it follows that  
$(a+1)^{m} = L + U < \left(\frac 1a +1\right)U$. 

\noindent Since $a+1 = \frac 1p$, we conclude that
$$
U > \frac{(a+1)^m}{\frac1a + 1} = \frac {1-p}{p^{m}}.
$$
\noindent {\bf Case 2: $m$ is even.} The proof is similar to the case when $m$ is odd.
Since $m$ is even it can be written as $m = 2d$, for some integer $d \geq 1$ so that $\lceil m/2 \rceil = d$. Let $a = \frac 1p -1$ and observe that $a>1$, since $p < \frac 12$.
From the binomial theorem we have that
\begin{align}
(a+1)^{m}
&= \notag \sum_{i =0 }^{m} {m \choose i} a^i \\
&= \notag \sum_{i=0 }^{d-1}  {m \choose i} a^i  +
\sum_{i = d }^{m} {m \choose i} a^i \\
&= \label{majeven:rec2} L'+U',
\end{align}
where $L'$ and $U'$ are defined as follows
\begin{align}
L' & :=  \label{compare0even:eq}
\sum_{i=0 }^{d-1}  {m \choose i} a^i = \sum_{i=1}^{d} {m \choose d-i} a^{d-i},  \mbox{ and }\\
U' &:= \label{compareeven:eq}
\sum_{i = d}^{m} {m \choose i} a^i \ = \sum_{i=0 }^{d} {m \choose d+i} a^{d+i}.
\end{align}

\noindent Now we compare summands in $L'$ and $U'$, namely
$$
{m \choose d-i} a^{d-i} \mbox{ and } {m\choose d+i} a^{d+i} ,
$$
for $i=0,1,\ldots ,d$. In Formulas~\eqref{compare0even:eq}-\eqref{compareeven:eq} observe that the left term when multiplied by
$a^{2i}$ is equal to the right term, namely $a^{2i} {m \choose d-i} a^{d-i} =  {m\choose d+i} a^{d+i}$, for $i=0,1,\ldots ,d$.
Since $a> 1$ and $d \geq 1$ we conclude that
\begin{align}
aL' &\leq \label{compare2:eq}
\sum_{i=0 }^{d} a {m \choose d-i} a^{d-i} < \sum_{i=0 }^{d} a^{2i} {m \choose d-i} a^{d-i} = U'.
\end{align}
From Equations~\eqref{majeven:rec2}~and~\eqref{compare2:eq}, it follows that  $(a+1)^{m} = L + U <
\left(\frac 1a +1\right)U$. Since $a+1 = \frac 1p$, we conclude that
$$
U' > \frac{(a+1)^m}{\frac1a + 1} = \frac {1-p}{p^{m}}.
$$

Therefore, Inequality~\eqref{maj:rec1} is proved in both cases of $m$ odd and $m$ even. Thus, the proof of Lemma~\ref{lemma:plowerthanonehalf} is complete.
\end{proof}

\noindent We may now conclude the following.
\begin{thm}
\label{maj:thm1}
The majority rule applied to a swarm of $m$ \gls*{mavs} executing  Algorithm~\ref{alg:rec} reduces the probability of error of the recognition process as long as $p$ is less than $1/2$.
\end{thm}
\begin{proof}
Let $m$ be the number of \gls*{mavs}.
Therefore $1 - p_m$ is the probability that the majority is at least
composed of $\lceil m/2 \rceil$ \gls*{mavs} correctly performing recognition, i.e.,
\begin{align}
1 - p_m &= \notag
\sum_{i= \lceil m/2 \rceil}^{m} {m \choose i} (1-p)^i p^{m-i} \\
&= \label{maj:rec3}  p^{m} \sum_{i=\lceil m/2 \rceil}^{m} {m\choose i} \left(\frac1p -1 \right)^i  .
\end{align}
Now, for $p < 1/2$ Lemma~\ref{lemma:plowerthanonehalf} says that
\begin{align}
1 - p &<  \label{maj:rec4}
p^{m} \sum_{i=\lceil m/2 \rceil}^{m}  {m \choose i} \left(\frac1p -1 \right)^i ,
\end{align}
which in view of Equation~\eqref{maj:rec3} implies that $p_m < p$, i.e., the probability of error for a swarm of $m$ \gls*{mavs} is less than
for \gls*{mav} in solo. This proves the theorem.
\end{proof}

\noindent A similar proof also yields the following.
\begin{thm}
\label{maj:thm2}
The majority rule applied to a swarm of $m$
\gls*{mavs} executing Algorithm~\ref{alg:adv} reduces the probability
of error of the advice process as long as $q < 1/2$,
\end{thm}
\begin{proof}
The proof is similar to the proof of Theorem~\ref{maj:thm2}.
\end{proof}

Note that there are additional possibilities in
Algotithm~\ref{alg:adv}. The \gls*{mavs} in a swarm could also acquire information
either from the same landmark or from different landmarks (although we
do not investigate the latter case further).

\subsection{Approximating the majority}
\label{sec:approximating}

Let $S_m$ be the sum of $m$ mutually independent random variables each taking the value $1$ with probability $p$ and the value $0$ with probability $1-p$ (i.e., Bernoulli random trials). The majority probability discussed above is given by the formula $\Pr [S_m \geq \lceil \frac m2 \rceil]$. Good approximations of the majority probability for large values of $m$ can be obtained from the central limit theorem which states that
\begin{equation}
\label{centralimit:thm}
\footnotesize{
\Pr \left[ a \leq \frac{S_m -mp}{\sqrt{mp(1-p)}} \leq b \right] \to \frac1{\sqrt{2\pi}} \int_a^b e^{-x^2 /2} dx.
\mbox{ as $m \to \infty$}
}
\end{equation}
(see e.g.,~\cite{rozanov2013probability}).
For example, for any $m$ we have that
$$
S_m \geq \left\lceil \frac m2 \right\rceil  \Leftrightarrow
\frac{S_m -mp}{\sqrt{mp(1-p)}} \geq \frac{\left\lceil \frac m2 \right\rceil -mp}{\sqrt{mp(1-p)}} .
$$
Hence, the central limit theorem~\eqref{centralimit:thm} is applicable with $a = \frac{\left\lceil \frac m2 \right\rceil -mp}{\sqrt{mp(1-p)}} $ and $b = +\infty$, where $p<1/2$ is a constant..

\section{Experiments and Simulations}
\label{sec:sim}

\noindent There is an interesting tradeoff between the majority probability $p_m$ and
cost of using a swarm of $m$ \gls*{mavs}. This helps put the
probabilistic gains in context w.r.t. the energy consumption and time
costs of the swarm.

\subsection{Cost measures and tradeoffs}
\label{sec:cost}

From Theorem~\ref{maj:thm1}, we know that for any  number $m$ of \gls*{mavs} the error probability is reduced from $p$ to $p_m$, similarly for Theorem~\ref{maj:thm2}. We now examine quantitative estimates of this error reduction in relation to specific numbers of \gls*{mavs} employed.

From Equation~\eqref{maj:rec3}, observe that we can derive the following identity expressing the ratio of improvement of the probability of correctness:
\begin{align}
\frac{1-p_m}{1-p}
&= \label{majopt:rec}
\sum_{i = \lceil m/2 \rceil }^{m}  {m \choose i} (1-p)^{i-1} p^{m-i}
\end{align}
In a way, one can think of the right-hand side of
Equation~\eqref{majopt:rec} as the ``fractional gain'' in the
correctness probability (because we are employing a majority rule)
that improves from $1-p$ to $1-p_m$. In general, we would like on
the one hand to ensure that $\frac{1-p_m}{1-p} > 1$ and on the other
hand optimize the right-hand side of Equation~\eqref{majopt:rec}. Since we are also interested in applying the majority algorithms for a relatively small
number of \gls*{mavs}, we give precise estimates for $m =2,3,4,5,6$ and $7$.

\begin{thm}
Table~\ref{tab:fractionalgainvalues} shows values for a fractional gain of $\frac{1-p_m}{1-p}$, for $m$ equal to 2, 3, 4, 5, 6 and 7~\gls*{mavs}.

\medskip
\begin{table}
\caption{Fractional gain values.}\label{tab:fractionalgainvalues}
\centering
\includegraphics[width=12cm]{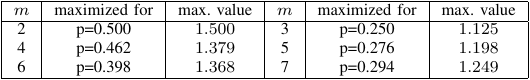}
\end{table}

\end{thm}
%
%
%
\begin{proof}
For $m = 2$ \gls*{mavs}, we can show that the ratio
  $\frac{1-p_2}{1-p} = \frac{1-p^2}{1-p} = 1+p$ is maximized for $p=1/2$ . 
Its maximum value  is $1+1/2 = 1.5$.

For $m = 3$, the ratio
  $\frac{1-p_3}{1-p}$ is maximized for $p=1/4$.
 Its maximum value
  is $1+1/8 = 1.125$. Indeed, calculations show that for $m=3$ the
  righthand side of Equation~\eqref{majopt:rec} is equal to $1 +p
  -2p^2$. Calculations also show that $1 +p -2p^2 $ is maximized when
  $p=1/4$ and attains the maximum value $1+1/8$. Hence also
  $\frac{1-p_3}{1-p}$ is maximized when $p=1/4$ and attains the
  maximum value $1+1/8 = 1.125$.

For $m = 4$, we have $\frac{1-p_4}{1-p} = (1-p) (3p^2+2p+1)$,
maximized for $p = \frac{1+\sqrt{10}}9$

For $m=5$, $\frac{1-p_5}{1-p} =
  10(1-p)^2p^2 + 5(1-p)^3p +(1-p)^4$. The derivative of the righthand
  side with respect to $p$ is equal to $24 p^3 -27 p^2 +2p +1$. One of
  the roots of this polynomial is $p=1$ and therefore $24 p^3 -27 p^2
  +2p +1 = (p-1)(24p^2 -3p+1)$. The positive root of the quadratic
  $24p^2 -3p+1$ is equal to $p= \frac{3+\sqrt{9+96}}{48} =
  \frac{3+\sqrt{105}}{48} \approx 0.275978$ and attains the maximum
  value $1.1917$.

For $m = 6$, $\frac{1-p_6}{1-p} = (1-p)^2 (20p^3+15(1-p)p^2+6(1-p)^2)p+(1-p)^3)$. This is maximized for $1.368$.

For $m=7$, $\frac{1-p_7}{1-p} =
  (1-p)^3 (35 p^3+35 (1-p)p^2+7 (1-p)^2 \cdot p+(1-p)^3).$ The
  derivative of the righthand side above is $(1-p)^2 (-42 \cdot
  p^3-102 \cdot p^2+32 \cdot p+1)$ which yields the root $p \approx 0.294$
  and attains the maximum value $1.249$.
\end{proof}

\begin{table}[!hptb]
  \caption{Left to right columns provide (1) the number of \gls*{mavs},
   (2) the error probability ($p$), (3) the fractional gain $\frac{1-p_m}{1-p}$
from $m$ = 2 to $m$ = 7, and (4) the corresponding majority error.}
  \label{gain1:tbl}
  \centering
  \includegraphics[width=12cm]{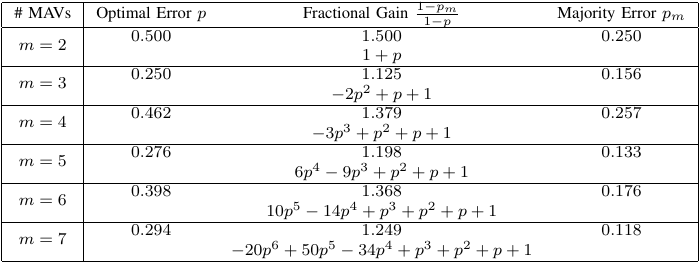}
\end{table}%

\begin{table}[!t]
  \centering
  \caption{Left to right columns provide (1) the number of \gls*{mavs},
   (2) the error probability ($p$), (3) the fractional gain $\frac{1-p_m}{1-p}$ from m = 9 to m = 21, and (4) the corresponding majority error.}
  \label{gain2:tbl}
  \includegraphics[width=12cm]{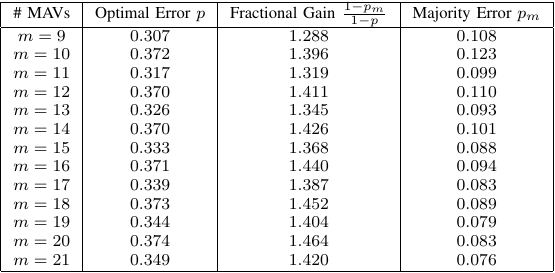}
\end{table}%

Table~\ref{gain1:tbl} displays the polynomials modeling the fractional gains for $m=2, 3, \ldots  7$. The improvement provided in
Theorem~\ref{maj:thm1} is more substantial when the number $m$
of \gls*{mavs} gets larger. This is also confirmed by the calculations above.
Table~\ref{gain2:tbl} displays the optimal error
probability and fractional gain and the last column the majority error
probability $p_m$ for a given number $m$ of \gls*{mavs}, where $m \leq 21$.
Figure~\ref{fig:table_m} (a) plots the evaluation of equation
$\frac{1-p_m}{1-p}$ from $m = 2$ to $m = 20$ and Figure~\ref{fig:table_m} (b) from $m = 3$ to $m = 21$. The resulting curve
indicates the maximum value of $\frac{1-p_m}{1-p}$ for $p$.

\begin{figure}[!t]
    \centering
    \subfigure[Even numbers of \gls*{mavs}]{
      \includegraphics[width=6.5cm]{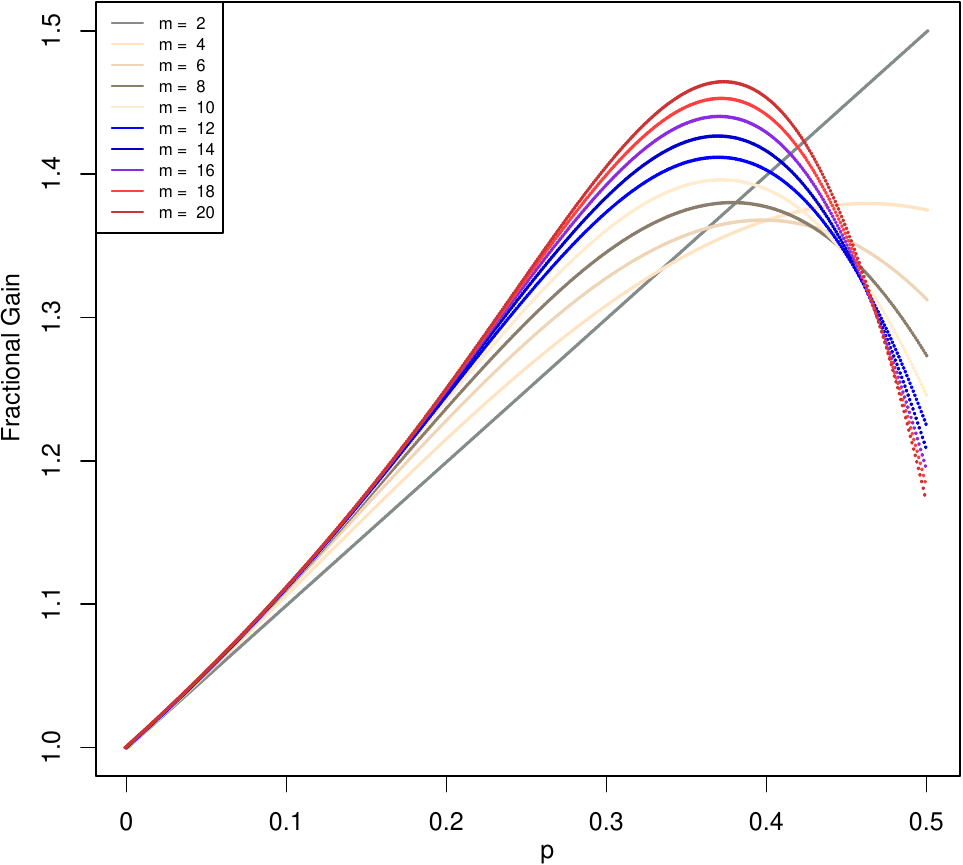}
    }
    \subfigure[Odd number of \gls*{mavs}]{
      \includegraphics[width=6.5cm]{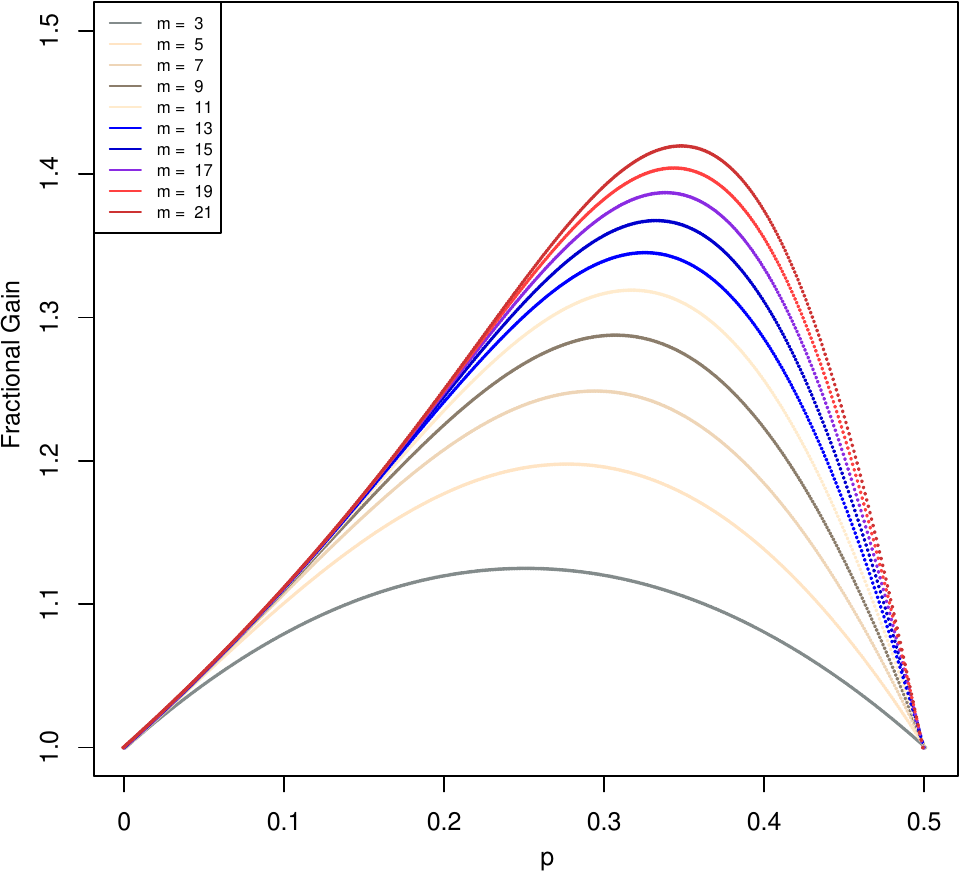}
    }
   \caption{Plots of the fractional gain function $\frac{1-p_m}{1-p}$ for varied $p$'s.}
\label{fig:table_m}
\end{figure}

\subsection{Numerical Simulations}
\label{sec:simBack}


\noindent Algorithms~\ref{alg:rec} and~\ref{alg:adv} have been integrated
into a Java simulator, which implements swarm populations modeled as mobile
agents. Each swarm executes the algorithms within a terrain
of interconnected landmarks. It consists of a
simple discrete event, time-step based simulation engine, in which
the swarm executes our algorithms at every step
of simulated time. The simulation engine implements a discrete event
scheduler, a graphical view, a data collection system, and the
simulated objects themselves, i.e., landmarks and agents. Videocaptures and
 source code are available online, at \href{http://j.mp/mavsim}{http://j.mp/mavsim} and \href{https://github.com/jgalfaro/mirrored-scavesim/tree/master/mavsim}{GitHub}.

Using our Java simulation, we validate five different scenarios. Each scenario 
relates the number of \gls*{mavs} with the error probability of the majority 
rule varying the {\em recognition} and {\em advice} error ratios between 70\%, 80\%, and 90\% (cf. Section~\ref{sec:algo}). Figures~\ref{fig:scenarios}(a,b) represent two traditional grid structures of MAVSIM (i.e., a $10\times 10$-grid and a $25\times 25$-grid). Figures~\ref{fig:scenarios}(c,d,e) represent three additional structures exported using the \href{https://www.openstreetmap.org}{OpenStreetMap} online service. More precisely, Figure~\ref{fig:scenarios}(a) shows a $10\times 10$-grid structure of MAVSIM; Figure~\ref{fig:scenarios}(b) a $25\times25$-grid structure; Figure~\ref{fig:scenarios}(c) a topological structure exported from \emph{OpenStreetMap} using \emph{Carleton University} as location; Figure~\ref{fig:scenarios}(d) a topological structure exported from \emph{OpenStreetMap} using Telecom SudParis (at the NanoInnov center of the campus of the Institut Polytechnique de Paris (IPP) and the Paris-Saclay University) as location; Figure~\ref{fig:scenarios}(e) a topological structure exported from \emph{OpenStreetMap} using University of Campinas  as location.

\begin{figure}[!t]
\centering
\subfigure[$10\times10$-grid]{
\includegraphics[width=0.3\linewidth]{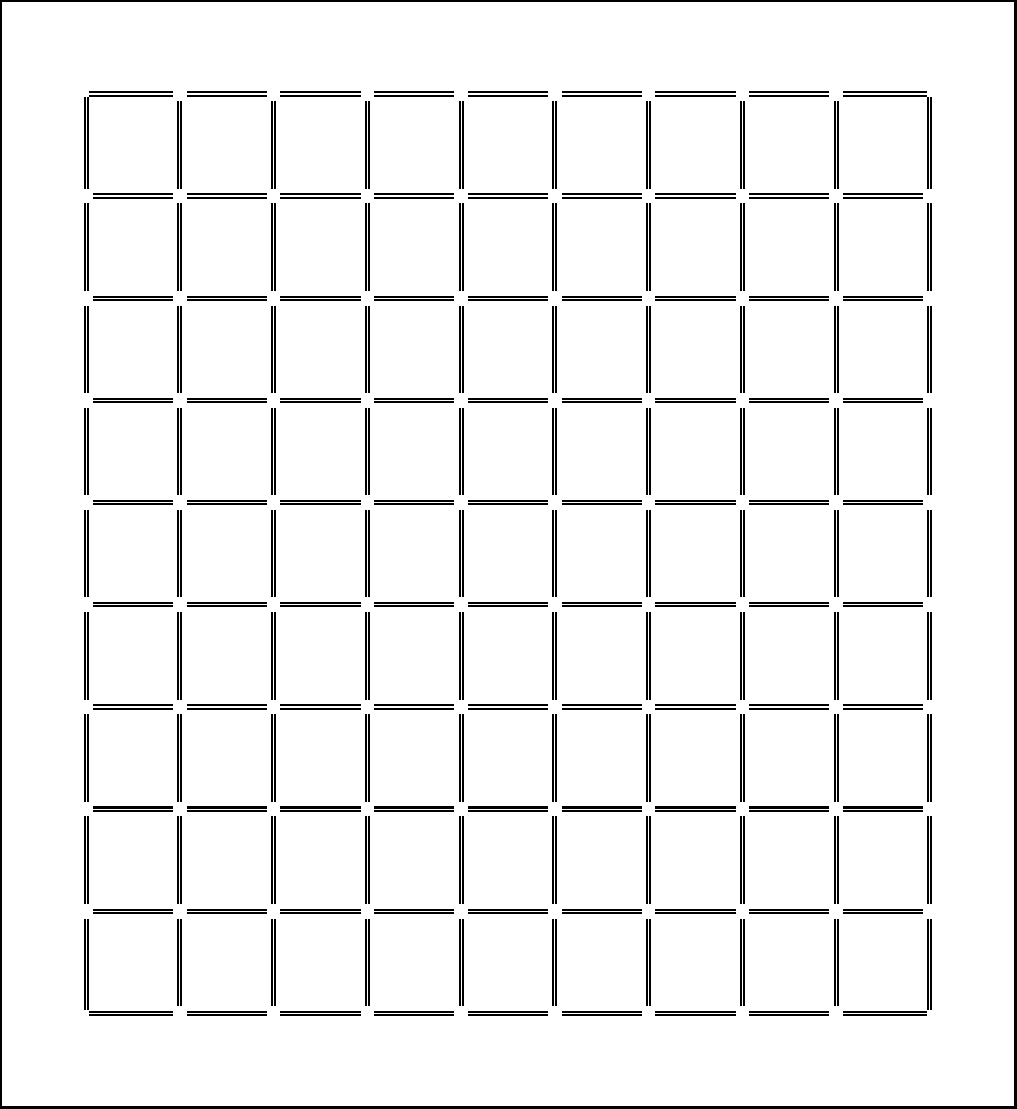}
}
\subfigure[$25\times25$-grid]{
  \includegraphics[width=0.3\linewidth]{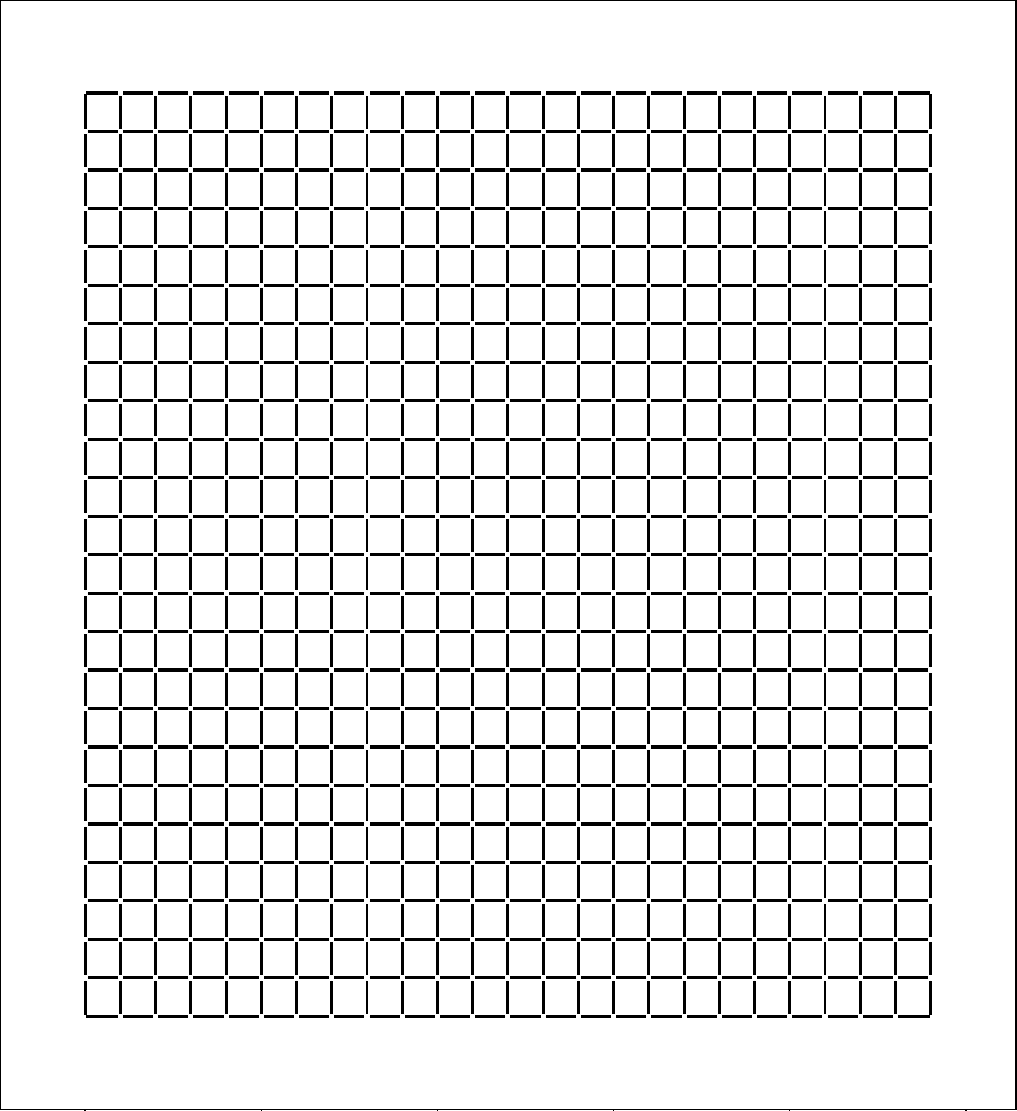}
}\\
\subfigure[Carleton University]{
  \includegraphics[width=0.27\linewidth]{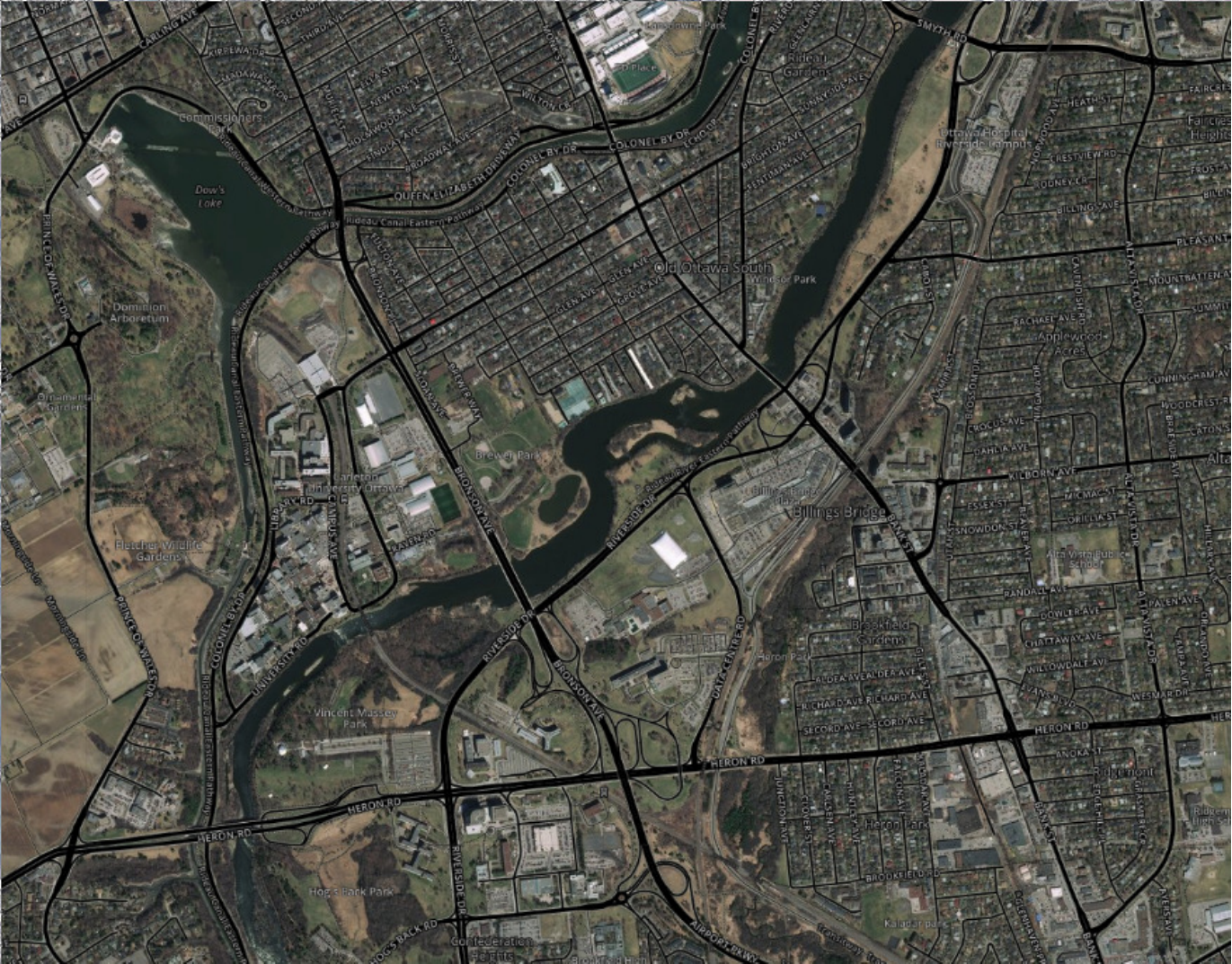}
}
\subfigure[IPP -- T. SudParis]{
  \includegraphics[width=0.31\linewidth]{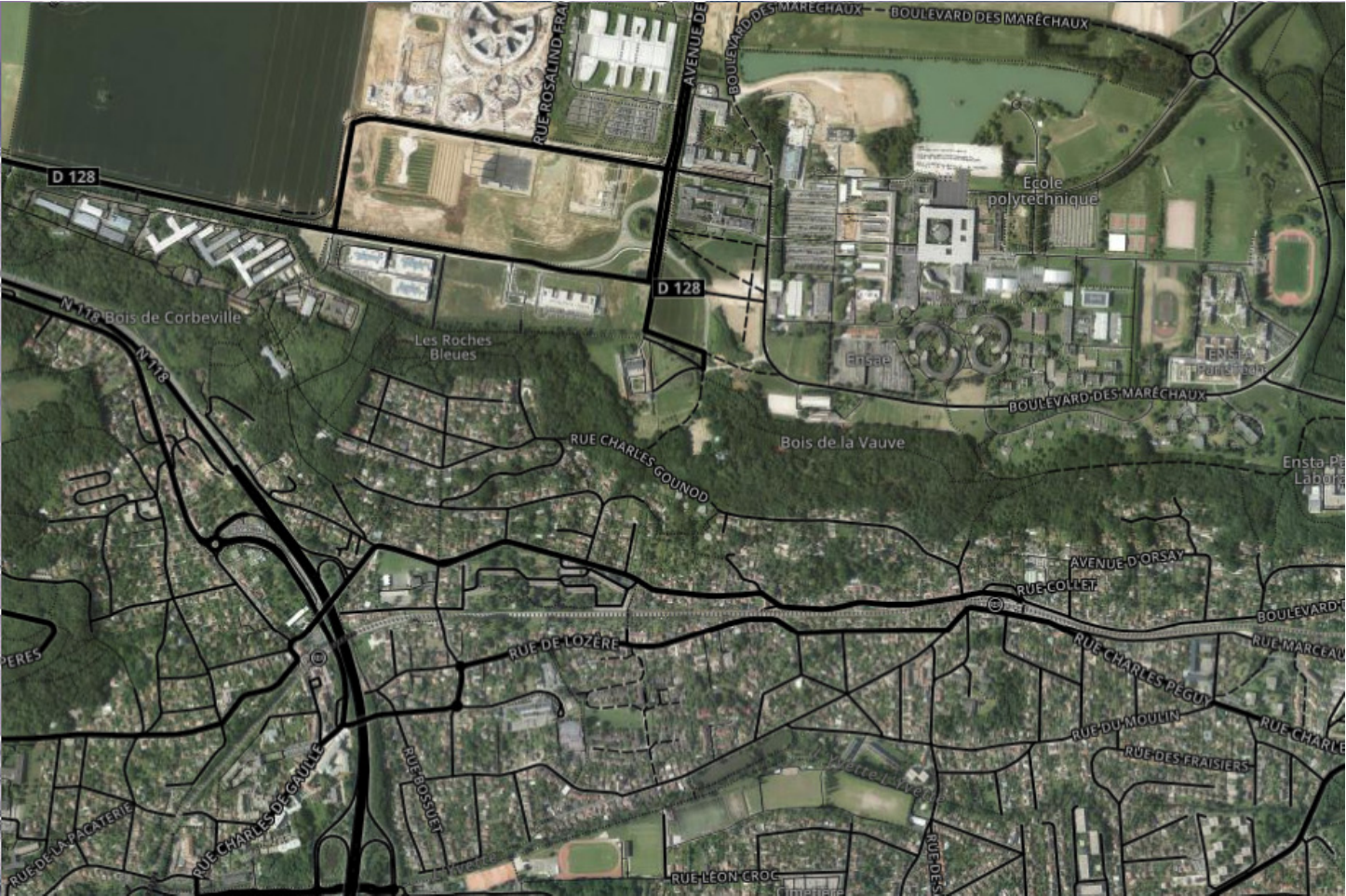}
}
\subfigure[Campinas -- Unicamp]{
  \includegraphics[width=0.31\linewidth]{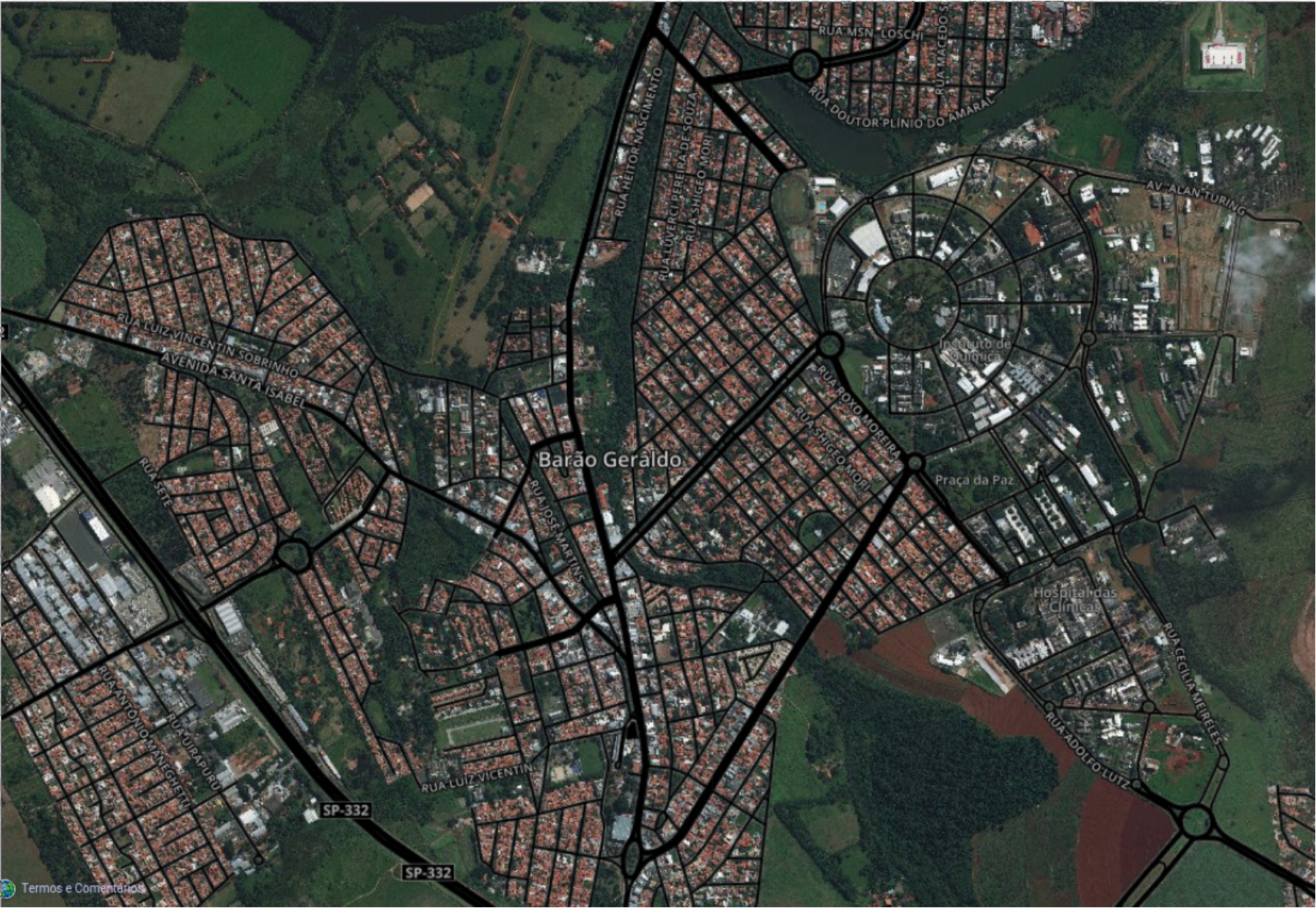}
}
\subfigure[Carleton University Map]{
  \includegraphics[width=0.27\linewidth]{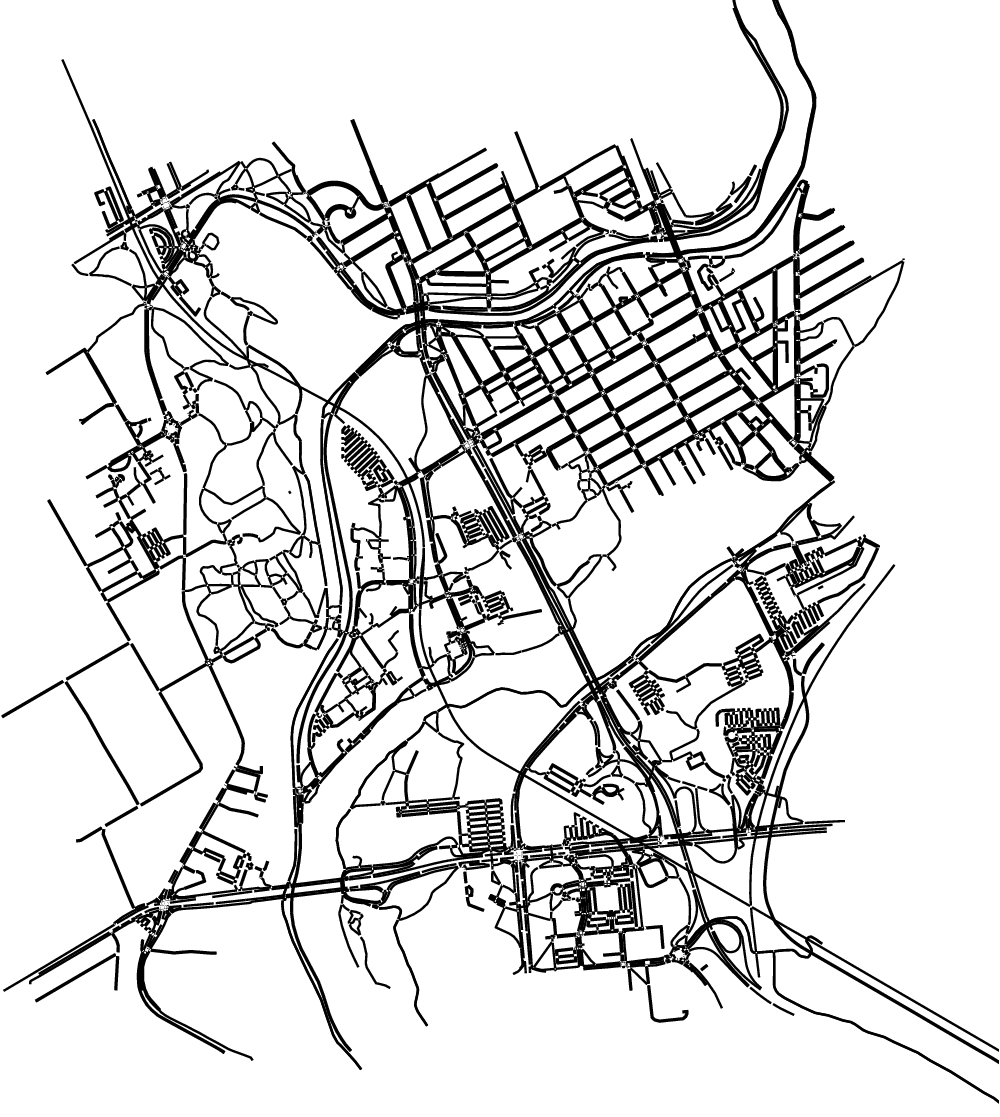}
}
\subfigure[IPP -- T. SudParis Map]{
  \includegraphics[width=0.31\linewidth]{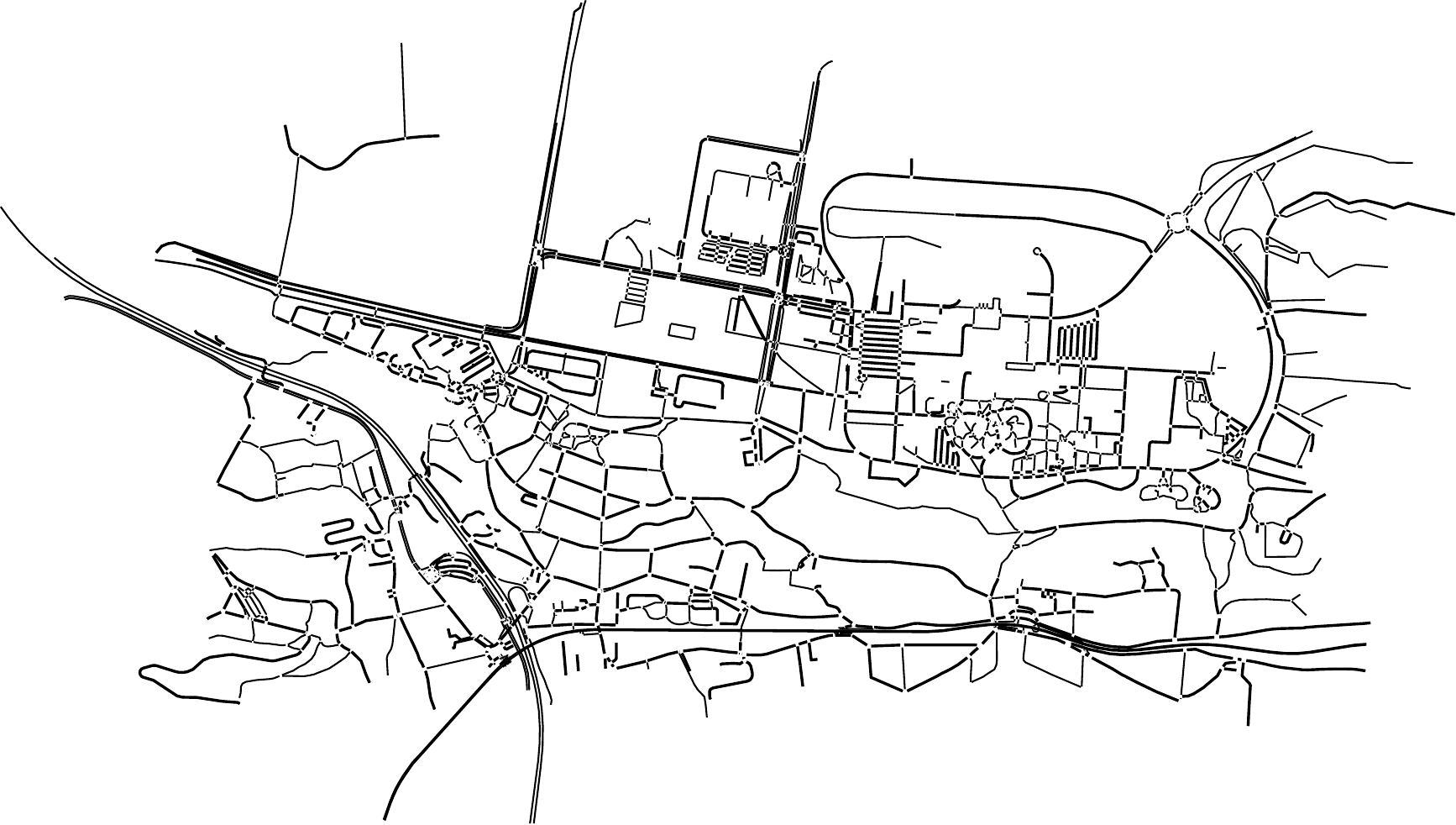}
}
\subfigure[Campinas -- Unicamp Map]{
  \includegraphics[width=0.3\linewidth]{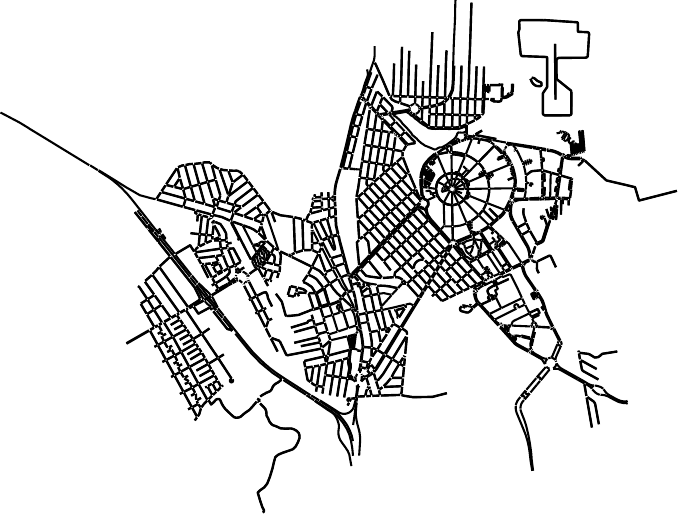}
}
\caption{Sample MAVSIM scenarios for the experimental results.}\label{fig:scenarios}
\end{figure}

\begin{figure*}[]%
 \centering
 \subfigure[$10\times 10$-grid]{\includegraphics[width=0.25\linewidth]{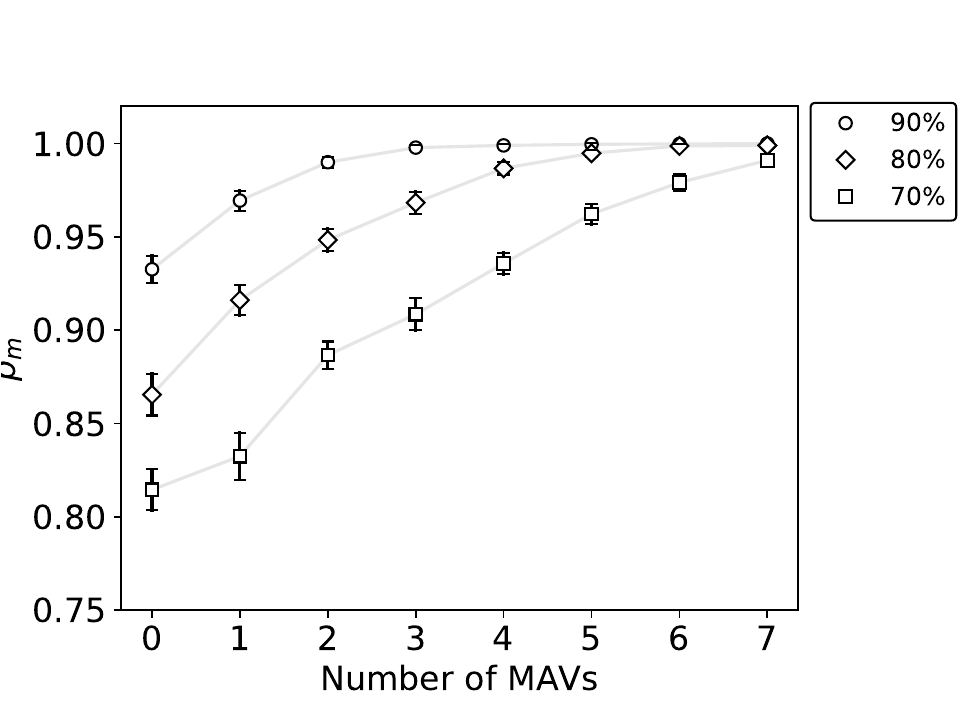}\label{fig:grid10}}
 \subfigure[$25\times 25$-grid]{\includegraphics[width=0.25\linewidth]{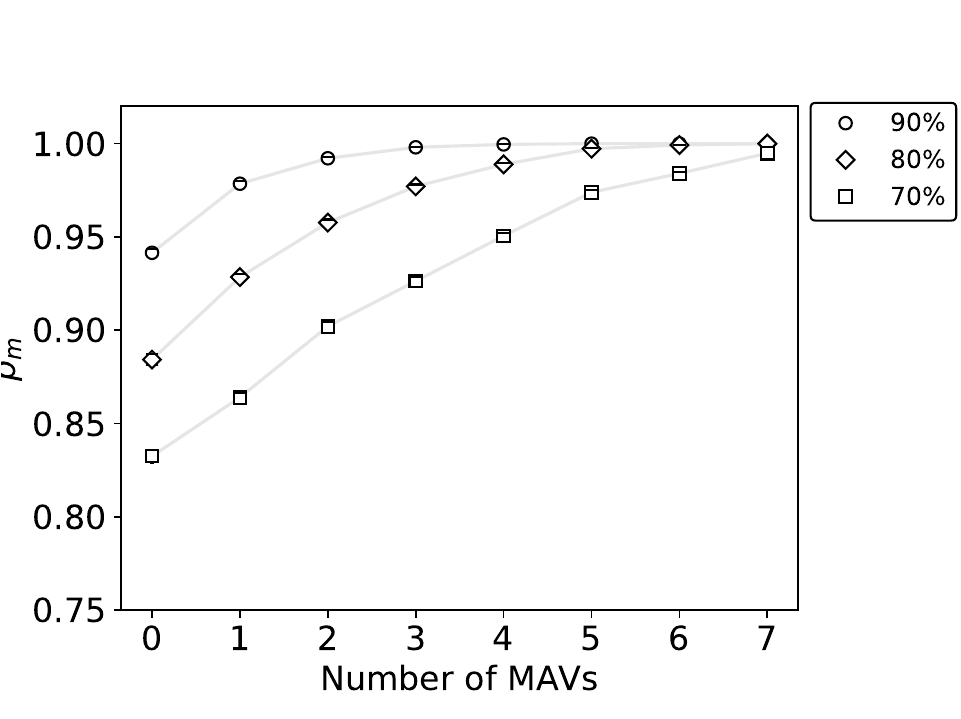}\label{fig:grid25}}
 \subfigure[Carleton University]{\includegraphics[width=0.25\linewidth]{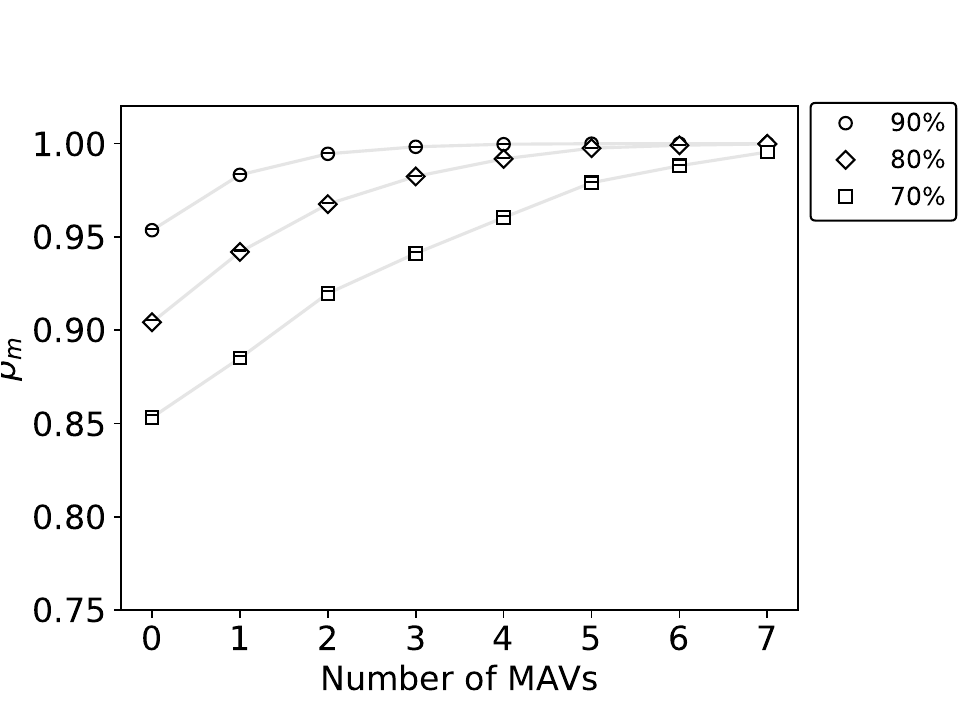}\label{fig:carleton}}\\
 \subfigure[IPP -- T. SudParis]{\includegraphics[width=0.25\linewidth]{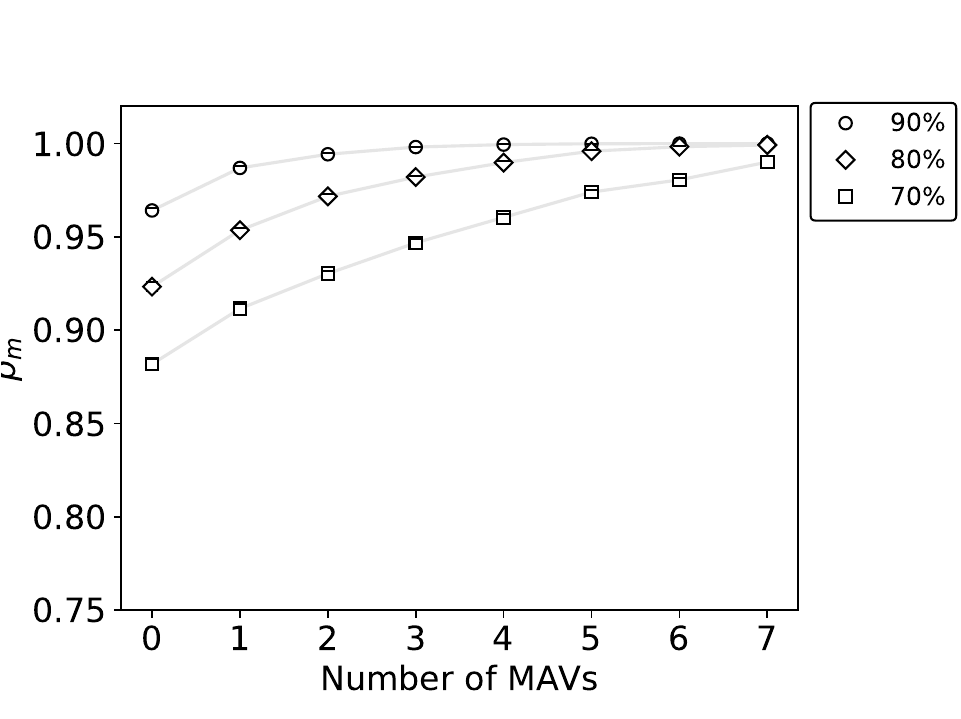}\label{fig:paris}}
 \subfigure[Campinas -- Unicamp]{\includegraphics[width=0.25\linewidth]{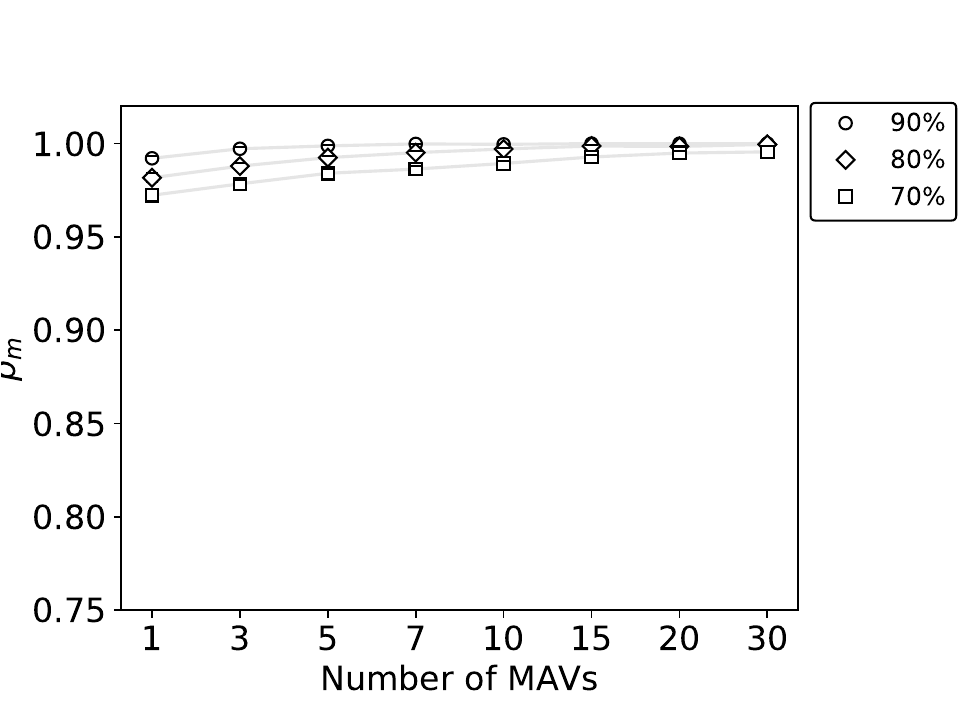}\label{fig:unicamp}}
 \caption{Relation between the number of \gls*{mavs} and error probability of the majority rule varying the {\em recognition} and {\em advice} error ratios between 70\% (squares), 80\% (diamonds), and 90\% (circles). The vertical axis represents the error probability. The horizontal axis represents the number of \gls*{mavs}.}%
 \label{fig:results}%
\end{figure*}

\begin{figure*}[]%
 \centering
  \subfigure[$10\times 10$-grid]{\includegraphics[width=0.25\linewidth]{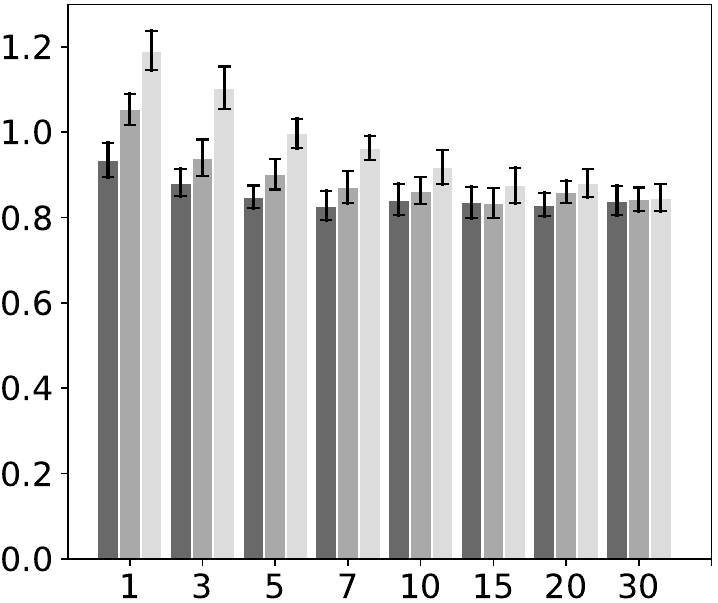}\label{fig:grid10-consumption}}
 \subfigure[$25\times 25$-grid]{\includegraphics[width=0.255\linewidth]{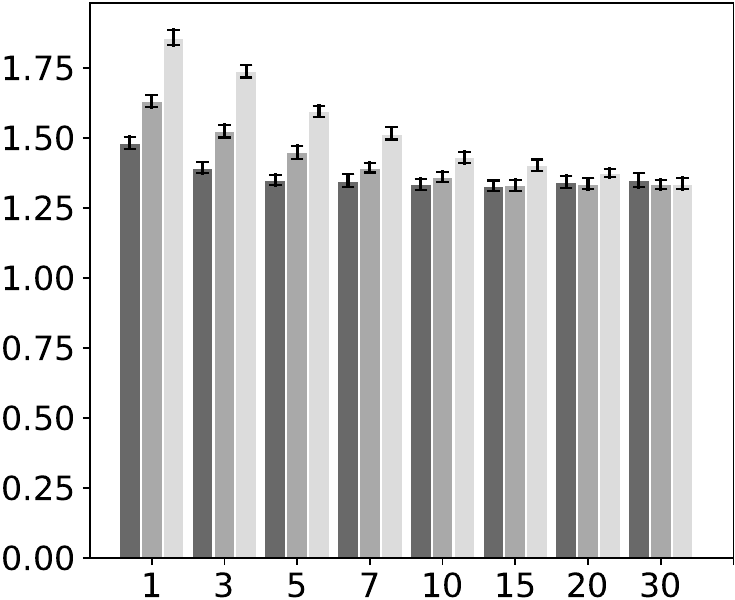}\label{fig:grid25-consumption}}
 \subfigure[Carleton University]{\includegraphics[width=0.25\linewidth]{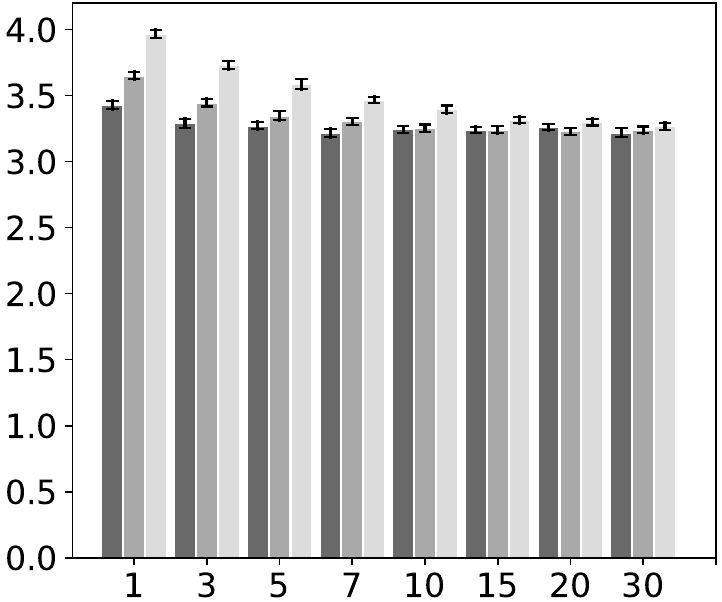}\label{fig:carleton-consumption}}\\
 \subfigure[IPP -- T. SudParis]{\includegraphics[width=0.25\linewidth]{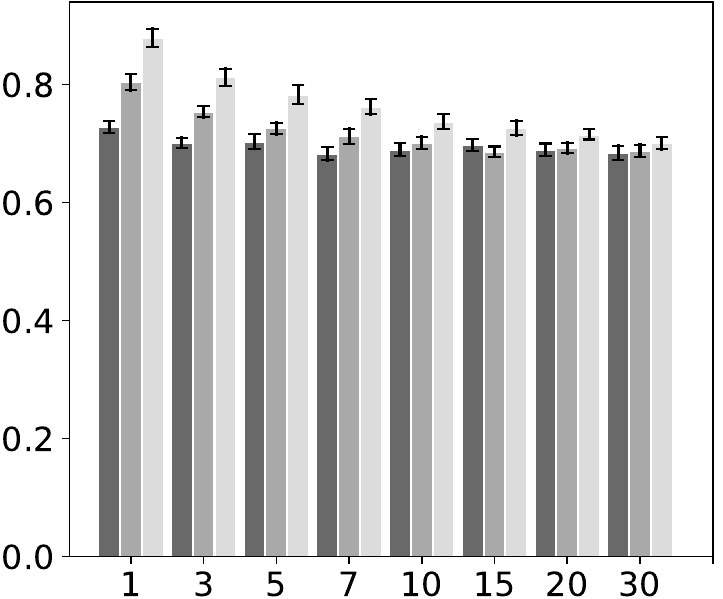}\label{fig:paris-consumption}}
 \subfigure[Campinas -- Unicamp]{\includegraphics[width=0.25\linewidth]{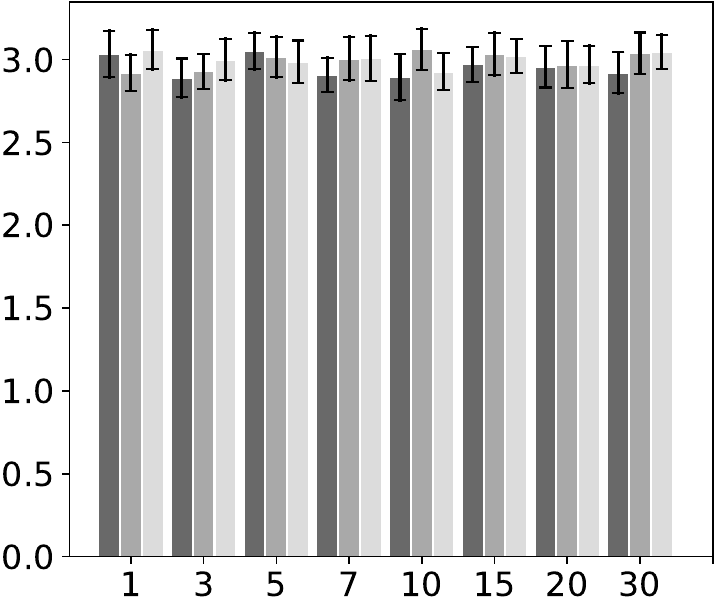}\label{fig:unicamp-consumption}}
 \caption{Relation between the number of \gls*{mavs} vs. energy consumption, when varying the {\em recognition} and {\em advice} error ratios from 70\% (light-gray), to 80\% (gray) and 90\% (dark-gray), and $s_{i}$ = 5 m/s. The vertical axis represents the energy consumption (joules). The horizontal axis represents the number of \gls*{mavs}.}%
 \label{fig:energy-consumption}%
\end{figure*}

\begin{table}[!b]
\caption{Characteristics of each scenario \label{tab:landmarks}}
\centering
  \includegraphics[width=12cm]{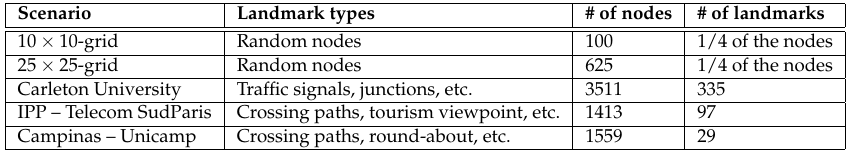}
\end{table}

\subsection{Performance Evaluation}
Table~\ref{tab:landmarks} shows some representative characteristics of each of the previous scenarios.  Figure~\ref{fig:results} shows the relation between the number of \gls*{mavs} and error probability of the majority rule varying the {\em recognition} and {\em advice} error ratios between 70\%, 80\%, and 90\% using, respectively, the $10\times 10$-grid, $25\times 25$-grid, 
Carleton, Telecom SudParis, and Unicamp scenarios. The results in Figure~\ref{fig:results} also show that the majority rule applied to a swarm of $m$ MAVs reduces the probability of error of the recognition process and our solution benefits from swarm cooperation.

\subsection{Energy Evaluation}

The energy consumed in a straight flight $E_{fly}$ to move a distance $d$ at a given speed $s_{i}$ can be computed as the integral of the power $P(s_{i})$ in function of the given speed $s_i$ along the time \cite{di2015energy}:

\begin{equation}
\label{eq:degree}
    E_{fly}(d,s_{i}) = { \int_{t=0}^{t=d/s_{i}} P(s_{i})d t = P(s_{i}) \frac{d}{s_{i}} }
\end{equation}

Figure~\ref{fig:energy-consumption}  shows the relating the number of \gls*{mavs} and the energy consumption varying the {\em recognition} and {\em advice} error ratios between 70\%, 80\%, and 90\% in the aforementioned scenarios ($10\times 10$-grid, $25\times 25$-grid, Carleton University, Telecom SudParis, and Unicamp, respectively). The results show that the energy consumption of the \gls*{mavs} reduces to approximately~$35\%$ when the number of \gls*{mavs} increases. 

\section{Conclusions}
\label{sec:conc}

\noindent We have presented an error tolerant path planning algorithm for
\gls*{mav} swarms. We have assumed a navigation system 
in which the \gls*{mavs} find their path by using their on board
cameras, by identifying and following series of visual landmarks. We
have assumed landmarks a priori information,
but which interpretation by the \gls*{mavs} is error prone. We have defined
two types of errors: (1) {\em recognition} errors, e.g., due to faulty
sensors which misinterpret the sensed data, and (2) {\em advice}
errors caused by the landmarks, e.g., due to weather conditions or
outdated information. Our
solution benefits from swarm cooperation. If the \gls*{mavs} in the
swarm can communicate and exchange information, then the {\em recognition}
and {\em advice} error ratios {\em get minimized to one fourth} with the cost
of {\em increasing the total number of \gls*{mavs} by twenty}. We have
validated our proposal with appropriate simulations, implemented over
a Java simulator available at
\href{http://j.mp/mavsim}{http://j.mp/mavsim}. 

The recognition and advice algorithms presented are based on binary decision making. An interesting setting worth exploring is non-binary. For example, consider a swarm of five \gls*{mavs} required to decide for its next move based on the majority color of a door (say, among Yellow (Y), Green (G), Blue (B)); in this case it is assumed that Y is to be positive outcome while G, B negative outcomes. On the one hand, if the respective outputs of the visual systems are Y, G, Y, Y, B then Yellow Door is considered to be the positive outcome and occurs three out of five times. On the other hand, if the respective outputs of the visual systems are Y, G, Y, B, B, then there is no majority of identical colors. In particular, majority can be formed by three identical answers. Such situations can be handled using voting schemes and fuzzy logic, which would be the focus of future research.

The basic idea of our algorithms is to enhance quality of recognition and advice by having multiple \gls*{mavs} make a decision after exchanging information they have obtained. Naturally, this increases the cost of movement since multiple \gls*{mavs} will be traveling to a destination. Therefore, it would be interesting to look at trade-offs of the cost of the search that take into account either time or total energy versus the number of \gls*{mavs} in the swarm for a given  budget.

\appendixtitles{no} 
%
%

\end{paracol}
\reftitle{References}

\end{document}